\documentclass[sigplan,screen]{acmart}

\usepackage{tikz}
\usetikzlibrary{
  cd,
  math,
  decorations.markings,
  decorations.pathreplacing,
  positioning,
  arrows.meta,
  circuits.logic.US,
  shapes,
  calc,
  fit,
  trees,
  quotes}
\usetikzlibrary{decorations.pathmorphing}
\usetikzlibrary{decorations.markings}
\usetikzlibrary{decorations.pathreplacing}
\usetikzlibrary{arrows}
\usetikzlibrary{shapes.geometric}
\usepackage{tikzit}
\usepackage{tikz-cd}

\usepackage{mathtools}
\usepackage{amsthm}
\usepackage{amsmath}
\usepackage[utf8]{inputenc}

\usepackage{hyperref}
\usepackage{pdfpages}
\usepackage{float}

\usepackage{dsfont}
\usepackage{relsize}
\usepackage{enumitem}

\usepackage{bibentry}

\newcommand{\ab}{\mathbf{A}}
\newcommand{\bb}{\mathbf{B}}
\newcommand{\cb}{\mathbf{C}}
\newcommand{\db}{\mathbf{D}}
\newcommand{\tb}{\mathbf{T}}
\newcommand{\ib}{\mathbf{I}}

\newcommand{\clq}{C_{\leq_{c}}}

\newcommand{\dab}{\mathbf{D_A}}

\newcommand{\dbb}{\mathbf{D_B}}
\newcommand{\dcb}{\mathbf{D_C}}

\newcommand{\fab}{\mathbf{F_A}}
\newcommand{\false}{\text{false}}
\newcommand{\fbb}{\mathbf{F_B}}
\newcommand{\fcb}{\mathbf{F_C}}

\newcommand{\lqs}{\leq_{S}}

\newcommand{\lgk}{Lan_{G}K}
\newcommand{\lgkl}{Lan_{G}K_L}

\newcommand{\metid}{\met_{id}}

\newcommand{\met}{\mathcal{\mathbf{Met}}}

\newcommand{\prtid}{\prt_{id}}

\newcommand{\prt}{\mathbf{Part}}

\newcommand{\px}{\mathcal{P}_X}

\newcommand{\py}{\mathcal{P}_Y}

\newcommand{\rgk}{Ran_{G}K}
\newcommand{\rgkr}{Ran_{G}K_R}

\newcommand{\rlp}{\rl_{\geq 0}}
\newcommand{\rl}{\mathbb{R}}

\newcommand{\set}{\mathbf{Set}}
\newcommand{\sofc}{S^1_{f_c}}
\newcommand{\sof}{S^1_{f}}

\newcommand{\stf}{S^2_{f}}

\newcommand{\true}{\text{true}}

\newcommand{\xrefined}{(X, \{\{x\}\ |\ x \in X\})}

\AtBeginDocument{%
  \providecommand\BibTeX{{%
    \normalfont B\kern-0.5em{\scshape i\kern-0.25em b}\kern-0.8em\TeX}}}

\setcopyright{none}
\copyrightyear{2022}
\acmYear{2022}

\acmConference[...]{...}{...}{...}
\acmBooktitle{...}

\begin{document}
\title{Kan Extensions in Data Science and Machine Learning}
\author{Dan Shiebler}
\email{danshiebler@gmail.com}
\affiliation{
  \institution{Kellogg College, University of Oxford}
  \country{United Kingdom}
}


\begin{abstract}
A common problem in data science is ``use this function defined over this small set to generate predictions over that larger set.'' Extrapolation, interpolation, statistical inference and forecasting all reduce to this problem. The Kan extension is a powerful tool in category theory that generalizes this notion. In this work we explore several applications of Kan extensions to data science. We begin by deriving a simple classification algorithm as a Kan extension and experimenting with this algorithm on real data. Next, we use the Kan extension to derive a procedure for learning clustering algorithms from labels and explore the performance of this procedure on real data. We then investigate how Kan extensions can be used to learn a general mapping from datasets of labeled examples to functions and to approximate a complex function with a simpler one.

\end{abstract}

\keywords{Kan Extension, Data Science, Machine Learning, Generalization, Extrapolation}
\maketitle

\section{Introduction}

A popular slogan in category theoretic circles, popularized by Saunders Mac Lane, is: ``all concepts are Kan extensions'' \cite{maclane:71}. While Mac Lane was partially referring to the fundamental way in which many elementary category theoretic structures (e.g. limits, adjoint functors, initial/terminal objects) can be formulated as Kan extensions, there are many applied areas that have Kan extension structure lying beneath the surface as well. 

One such area is data science/machine learning. It is quite common in data science to apply a model constructed from one dataset to another dataset. In this work we explore this structure through several applications.

Our contributions in this paper are as follows:
\begin{itemize}
    \item We derive a simple classification algorithm as a Kan extension and demonstrate experimentally that this algorithm can learn to classify images.
    \item We use Kan extensions to derive a novel method for learning a clustering algorithm from labeled data and demonstrate experimentally that this method can learn to cluster images. 
    \item We explore the structure of meta-supervised learning and use Kan extensions to derive supervised learning algorithms from sets of labeled datasets and trained functions.
    \item We use Kan extensions to characterize the process of approximating a complex function with a simpler minimum description length (MDL) function.
\end{itemize}
The code that we use in this paper be found at:
\begin{center}
https://anonymous.4open.science/r/Kan\_Extensions-4C21/. 
\end{center}

\subsection{Related Work}


While many authors have explored how an applied category theoretic perspective can help exploit structure and invariance in machine learning \cite{shieblersurvey2021}, relatively few authors have explored applications of Kan extensions to data science and machine learning. 

That said, some authors have begun to explore Kan extension structure in topological data analysis. For example, Bubenik et al. \cite{bubenikinterpolation} describe how three mechanisms for interpolating between persistence modules can be characterized as the left Kan extension, right Kan extension, and natural map from left to right Kan extension. Similarly, McCleary et. al. \cite{mccleary2021edit} use Kan extensions to characterize deformations of filtrations. Furthermore, Botnan et. al. \cite{Botnan_2018} use Kan extensions to generalize stability results from block decomposable persistence modules to zigzag persistence modules and Curry \cite{Curry13sheaves} use Kan extensions to characterize persistent structures from the perspective of sheaf theory.

Other authors have explored the application of Kan extensions to databases. For example, in categorical formulations of relational database theory \cite{spivak2015relational, schultz2017algebraic,  schultz2016algebraic}, the left Kan extension can be used for data migration. Spivak et. al. \cite{spivak2020fast} exploit the characterization of data migrations as Kan extensions to apply the chase algorithm from relational database theory to the general computation of the left Kan extension.

Outside of data science, many authors have applied Kan extensions to the study of programs. For example, Hinze et. al. \cite{DBLP:conf/mpc/Hinze12} use Kan extensions to express the process of optimizing a program by converting it to a continuation-passing style.  Similarly, Yonofsky \cite{yanofsky2013kolmogorov} uses Kan extensions to reason about the output of programs expressed in a categorical programming language. Many authors have also used Kan extensions to reason about operations on the category of Haskell types and functions \cite{paterson2012constructing}.

\subsection{Background on Kan extensions}\label{kanextensionsbackground}

In this work we assume readers have a basic familiarity with category theory. For a detailed introduction to the field check out ``Basic Category Theory'' \cite{leinster2016basic} or ``Seven Sketches in Compositionality'' \cite{fong2018seven}. 

Suppose we have three categories $\ab, \bb, \cb$ and two functors $G: \ab \rightarrow \bb, K: \ab \rightarrow \cb$ and we would like to derive the ``best'' functor $F: \bb \rightarrow \cb$. There are two canonical ways that we can do this:
\begin{definition}
The left Kan extension of $K$ along $G$ is the universal pair of the functor $\lgk: \bb \rightarrow \cb$ and natural transformation $\mu: K \rightarrow (\lgk \circ G)$ such that for any pair of a functor $M: \bb \rightarrow \cb$ and natural transformation $\lambda: K \rightarrow (M \circ G)$ there exists a unique natural transformation $\sigma: \lgk \rightarrow M$ such that $\lambda = \sigma_G \circ \mu$ (where $\sigma_{G}(a) = \sigma(G a)$).
\end{definition}
\begin{definition}
The right Kan extension of $K$ along $G$ is the universal pair of the functor $\rgk: \bb \rightarrow \cb$ and natural transformation $\mu: (\rgk \circ G) \rightarrow K$ such that for any pair of a functor $M: \bb \rightarrow \cb$ and natural transformation $\lambda: (M \circ G) \rightarrow K$ there exists a unique natural transformation $\sigma: M \rightarrow \rgk$ such that $\lambda = \mu \circ \sigma_G$ (where $\sigma_{G}(a) = \sigma(G a)$).
\end{definition}
\begin{center}
\begin{tikzcd}[column sep=1in,row sep=1in]
\bb \arrow[dotted]{dr}{F} & \\
\ab \arrow{u}{G} \arrow{r}{K} & \cb
\end{tikzcd}
\end{center}
Intuitively, if we treat $G$ as an inclusion of $\ab$ into $\bb$ then the Kan extensions of $K$ along $G$ act as extrapolations of $K$ from $\ab$ to all of $\bb$. If $\cb$ is a preorder then the left and right Kan extensions respectively behave as the least upper bound and greatest lower bounds of $K$.

For example, suppose we want to interpolate a monotonic function $K: \mathbb{Z} \rightarrow \rl$ to a monotonic function $F: \rl \rightarrow \rl$ such that $F \circ G = K$ where  $G: \mathbb{Z} \hookrightarrow \rl$ is the inclusion map (morphisms in $\mathbb{Z}, \rl$ are $\leq$):
\begin{center}
\begin{tikzcd}[column sep=1in,row sep=1in]
\rl \arrow{dr}{F} & \\
\mathbb{Z} \arrow{u}{G} \arrow{r}{K} & \rl
\end{tikzcd}
\end{center}
%
We have that $\lgk: \rl \rightarrow \rl$ is simply $K \circ floor$ and $\rgk: \rl \rightarrow \rl$ is simply $K \circ ceil$, where $floor, ceil$ are the rounding down and rounding up functions respectively.


\section{Applications of Kan Extensions}

To start, recall the definitions of the left and right Kan extensions from Section \ref{kanextensionsbackground}. We will explore four applications of Kan extensions to generalization in machine learning:
\begin{itemize}
    \item Section \ref{section:kan-classification}: Learn a classifier from a dataset of labeled examples.
    \item Section \ref{section:kan-clustering}: Learn a mapping from metric spaces $(X, d_X)$ to partitions of $X$.
    \item Section \ref{section:kan-meta-learning}: Learn a mapping from datasets of labeled examples to functions.
    \item Section \ref{section:kan-function-approximation}: Approximate a complex function with a simpler one.
\end{itemize}

In each of these applications we first define categories $\ab,\bb,\cb$ and a functor $K:\ab\rightarrow \cb$ such that $\ab$ is a subcategory of $\bb$ and $G: \ab \hookrightarrow \bb$ is the inclusion functor. Then, we take the left and right Kan extensions $\lgk, \rgk$ of $K$ along $G$ and study their behavior. Intuitively, the more restrictive that $\bb$ is (i.e. the more morphisms in $\bb$) or the larger that $\ab$ is (and therefore the more information that is stored in $K$) the more similar $\lgk, \rgk$ will be to each other.

\section{Classification}\label{section:kan-classification}
We start with a simple application of Kan extensions to supervised learning. Suppose that $\ib$ is a preorder, $\ib' \subseteq \ib$ is a subposet of $\ib$, $K: \ib' \rightarrow \{\false,\true\}$ is a not-necessarily monotonic mapping, and we would like to learn a monotonic function $\ib \rightarrow \{\false,\true\}$ that approximates $K$ on $\ib'$. That is, $K$ defines a finite training set of points $S = \{(x, K(x)) \ |\ x \in \ib' \}$ from which we wish to learn a monotonic function $F: \ib \rightarrow \{\false,\true\}$. Of course, it may not be possible to find a monotonic function that agrees with $K$ on all the points in $\ib'$.
\begin{center}
\begin{tikzcd}[column sep=1in,row sep=1in]
\ib  \arrow[dotted]{dr}{F} & \\
\ib' \arrow[hookrightarrow]{u}{G} \arrow{r}{K} & \{\false,\true\}
\end{tikzcd}
\end{center}
If we treat $\ib'$ as discrete category, then $K$ is a functor and we can solve this problem with the left and right Kan extensions of $K$ along the inclusion functor $G: \ib' \hookrightarrow \ib$.

\begin{proposition}\label{supervisedpreorderproposition}
The left and right Kan extensions of $K: \ib' \rightarrow \{\false,\true\}$ along the inclusion map $G: \ib' \hookrightarrow \ib$ are respectively: 
\begin{gather*}
\lgk: \ib \rightarrow \{\false,\true\}
\qquad
\rgk: \ib \rightarrow \{\false,\true\}
\\
%
%
\lgk(x) = \begin{cases}
\true & \exists x' \in \ib', x' \leq x, K(x') = \true \\
\false  & \text{else}
\end{cases}
\\
%
\rgk(x) = \begin{cases}
\false  & \exists x' \in \ib', x \leq x', K(x') = \false \\
\true & \text{else}
\end{cases}
\end{gather*}
(Proof in Section \ref{proof:supervisedpreorderproposition})
\end{proposition}

In the extreme case that $Ob(\ib') = \emptyset$, for $x \in \ib$ we have that: 
\begin{align*}
\lgk(x) = \\
\begin{cases}
\true & \exists x' \in \ib', x' \leq x, K(x') = \true \\
\false  & \text{else}
\end{cases} = 
\false
\end{align*}
and:
\begin{align*}
\rgk(x) =\\
\begin{cases}
\false & \exists x' \in \ib', x \leq x', K(x') = \false \\
\true & \text{else}
\end{cases} =
\true
\end{align*}
Similarly, in the extreme case that $Ob(\ib') = Ob(\ib)$ we have by the functoriality of $K$ that for $x \in \ib$ both of the following hold if and only if $K(x) = \true$. 
\begin{gather*}
   \exists x' \in \ib', x' \leq x, K(x') = \true
   \\
   \not\exists x' \in \ib', x \leq x', K(x') = \false
\end{gather*}
Therefore in this extreme case we have:
\begin{align*}
\lgk(x) = \rgk(x) = K(x)
\end{align*}

Now suppose that $\ib'$ contains at least one $x'$ such that $K(x') = \true$ and at least one $x'$ such that $K(x') = \false$. In this case $\lgk$ and $\rgk$ split $\ib$ into three regions: a region where both map all points to $\false$, a region where both map all points to $\true$, and a disagreement region.
Note that $\rgk$ has no false positives on $\ib'$ and $\lgk$ has no false negatives on $\ib'$. 

For example, suppose $\ib=\rl, \ib'=\{1,2,3,4\}$ and we have:
\begin{gather*}
    K(1)=\false
    \qquad
    K(2)=\false
    \\
    K(3)=\true
    \qquad
    K(4)=\true
\end{gather*}
Then we have that:
\begin{align*}
\lgk(x) = \\
\begin{cases}
\true & \exists x' \in \ib', x' \leq x, K(x') = \true \\
\false  & \text{else}
\end{cases} 
= \\
\begin{cases}
\true & x \geq 3 \\
\false  & \text{else}
\end{cases}
\end{align*}
and that:
\begin{align*}
\rgk(x) = \\
\begin{cases}
\false & \exists x' \in \ib', x \leq x', K(x') = \false \\
\true & \text{else}
\end{cases}
=\\
\begin{cases}
\true & x > 2 \\
\false  & \text{else}
\end{cases}
\end{align*}
In this case the disagreement region for $\lgk,\rgk$ is $(2,3)$ and for any $x \in (2,3)$ we have $\lgk(x) < \rgk(x)$.  

As another example, suppose $\ib=\rl, \ib'=\{5,6,7,8\}$ and we have:
\begin{gather*}
    K(5)=\false
    \qquad
    K(6)=\true
    \\
    K(7)=\false
    \qquad
    K(8)=\true
\end{gather*}
Then we have that:
\begin{align*}
\lgk(x) =\\
\begin{cases}
\true & \exists x' \in \ib', x' \leq x, K(x') = \true \\
\false  & \text{else}
\end{cases} 
= \\
\begin{cases}
\true & x \geq 6 \\
\false  & \text{else}
\end{cases}
\end{align*}
and that:
\begin{align*}
\rgk(x) = \\
\begin{cases}
\false & \exists x' \in \ib', x \leq x', K(x') = \false \\
\true & \text{else}
\end{cases}
=\\
\begin{cases}
\true & x > 7 \\
\false  & \text{else}
\end{cases}
\end{align*}
In this case the disagreement region for $\lgk,\rgk$ is $[6,7]$ and for any $x \in [6,7]$ we have $\rgk(x) < \lgk(x)$.

%
%
While this approach is effective for learning very simple mappings, there are many choices of $K$ for which $\lgk$ and $\rgk$ do not approximate $K$ particularly well on $\ib'$ and therefore the disagreement region is large. In such a situation we can use a similar strategy to the one leveraged by kernel methods \citep{hofmann2008kernel} and transform $\ib$ to minimize the size of the disagreement region.

That is, we choose a preorder $\ib^{*}$ and transformation $f: \ib \rightarrow \ib^{*}$ such that the size of the disagreement region for $Lan_{f \circ G}K \circ f, Ran_{f \circ G}K \circ f$ is minimized. 
\begin{center}
\begin{tikzcd}[column sep=1in,row sep=1in]
\ib \arrow{r}{f} & \ib^{*} \arrow[dotted]{d}{F} \\
\ib' \arrow[hookrightarrow]{u}{G} \arrow{r}{K} & \{\false,\true\}
\end{tikzcd}
\end{center}
For example, if $\ib^{*} = \rl^a$ we can choose $f$ to minimize the following loss:
\begin{definition}\label{definition:orderingloss}
Suppose we have a set $\ib' \subseteq \ib$ and function $K: \ib' \rightarrow \{\false,\true\}$ such that:
\begin{gather*}
	%
	\exists x', x'' \in \ib',
	K(x') = \true,
	K(x'') = \false
\end{gather*}
Then the ordering loss $l$ maps a function $f: \ib \rightarrow \rl^a$ to an approximation of the size of the disagreement region for $Lan_{f \circ G}K \circ f, Ran_{f \circ G}K \circ f$. Formally, we define the ordering loss $l$ to be:
\begin{gather*}
l: (\ib \rightarrow \rl^a) \rightarrow \rl
\end{gather*}
\begin{align*}\label{optimizationequation}
    l(f) = \sum_{i \leq a}
    \max(0,\ 
    \max\{f(x)[i] \ | \ x \in \ib', K(x) = \false\}
    -\\
    \min\{f(x)[i] \ | \ x \in \ib', K(x) = \true\}
    )
\end{align*}
where $f(x)[i]$ is the $i$th component of the vector $f(x)[i] \in \rl^a$.  
\end{definition}
We can show that minimizing the ordering loss $l$ will also minimize the size of the disagreement region:
\begin{proposition}\label{optimizationproposition}
%
The ordering loss $l$ (Definition \ref{definition:orderingloss}) is non-negative and is only equal to 0 when $\forall x \in \ib'$ we have:
\begin{gather*}
    K(x)=(Lan_{f \circ G}K \circ f)(x)=(Ran_{f \circ G}K \circ f)(x)
\end{gather*}
(Proof in Section \ref{proof:optimizationproposition})
\end{proposition}

It is relatively straightforward to minimize the ordering loss with an optimizer like subgradient descent \citep{convexopt}.\footnote{Example code at https://anonymous.4open.science/r/Kan\_Extensions-4C21/.}


    
In Table \ref{fashionmnistclassification} we demonstrate that we can use this strategy to distinguish between the ``T-shirt'' (false) and ``shirt'' (true) categories in the  Fashion MNIST dataset \citep{xiao2017/online}. Samples in this dataset have $784$ features (pixels), so we train a simple linear model $f: \rl^{784} \rightarrow \rl^{10}$ with Adam \citep{kingma2014adam} to minimize the ordering loss $l(f)$ over a training set that contains $90\%$ of samples in the dataset. We then evaluate the performance of $Lan_{f \circ G}K \circ f,Ran_{f \circ G}K \circ f$ over both this training set and a testing set that contains the remaining $10\%$ of the dataset. We look at two metrics over both sets.

\begin{definition}\label{definition:true-positive-rate}
The true positive rate is the percentage of all true samples (shirts) which the classifier correctly labels as true. This is also known as recall or sensitivity. The true positive rate is $1.0$ if and only if there are no false negatives.
\end{definition}
\begin{definition}\label{definition:true-negative-rate}
The true negative rate is the percentage of all false samples (T-shirts) which the classifier correctly labels as false. This is also known as specificity. The true negative rate is $1.0$ if and only if there are no false positives.
\end{definition}
As we would expect from the definition of Kan extensions, the map $Lan_{f \circ G}K \circ f$ has no false negatives and $Ran_{f \circ G}K \circ f$ has no false positives on the training set. The metrics on the testing set are in-line with our expectations as well: $Lan_{f \circ G}K \circ f$ has a higher true positive rate and $Ran_{f \circ G}K \circ f$ has a higher true negative rate.

\begin{table}
    \centering
    \begin{tabular}{|c|c|c|c|c|}
        \hline
        \multicolumn{1}{|p{1.5cm}|}{\centering Model}
        &
        \multicolumn{1}{|p{1.0cm}|}{\centering Dataset}
        &
        \multicolumn{1}{|p{2cm}|}{\centering True Positive Rate}
        &
        \multicolumn{1}{|p{2cm}|}{\centering True Negative Rate}
        \\
        \hline
        \hline
        Left Kan Classifier & Train &  $1.000\ (\pm 0.000)$ & $0.612\ (\pm 0.042)$ \\
        \hline
        Right Kan Classifier & Train &  $0.705\ (\pm 0.035)$ & $1.000\ (\pm 0.000)$ \\
        \hline
        Left Kan Classifier & Test  &  $0.815\ (\pm 0.020)$ & $0.593\ (\pm 0.044)$\\
        \hline
        Right Kan Classifier & Test &  $0.691\ (\pm 0.044)$ & $0.837\ (\pm 0.026)$\\
        \hline
    \end{tabular}
    \caption{True positive rate and true negative rate of the left Kan classifier $Lan_{f \circ G}K \circ f$ and the right Kan classifier $Ran_{f \circ G}K \circ f$ where $f$ is a linear map trained to minimize the ordering loss $l(f)$ (Definition \ref{definition:orderingloss}) on the Fashion-MNIST  ``T-shirt'' vs ``shirt'' task \citep{xiao2017/online}. 
    We run a bootstrap experiment by repeatedly selecting $9000$ training samples and $1000$ testing samples, running the training procedure, and computing true positive rate and true negative rate metrics. Mean and two standard error confidence bounds from $10$ such bootstrap iterations are shown.
    }
\label{table:fashionmnistclassification}
\end{table}

\section{Clustering with Supervision}\label{section:kan-clustering}

Clustering algorithms allow us to group points in a dataset together based on some notion of similarity between them. Formally, we can consider a clustering algorithm as mapping a metric space $(X, d_X)$ to a partition of $X$.

In most applications of clustering the points in the metric space $(X, d_X)$ are grouped together based solely on the distances between the points and the rules embedded within the clustering algorithm itself. This is an unsupervised clustering strategy since no labels or supervision influence the algorithm output. For example, agglomerative clustering algorithms like HDBSCAN \citep{mcinnes2017accelerated} and single linkage partition points in $X$ based on graphs formed from the points (vertices) and distances (edges) in $(X, d_X)$.

However, there are some circumstances under which we have a few ground truth examples of pre-clustered training datasets and want to learn an algorithm that can cluster new data as similarly as possible to these ground truth examples. We can define the supervised clustering problem as follows. Given a collection of tuples
\begin{align*}
S = \{(X_1, d_{1}, P_{1}), (X_2, d_{2}, P_{2}), \cdots, (X_n, d_{n}, P_{n})\}
\end{align*}
where each $(X_i, d_{i})$ is a metric space and $P_{i}$ is a partition of $X_i$, we would like to learn a general function $f$ that maps a metric space $(X, d_{X})$  to a partition $P_{X}$ of $X$ such that for each $(X_i, d_{i}, P_{i}) \in S$ the difference between $f(X_i, d_{i})$ and $P_{i}$ is small. 

We can frame this objective in terms of categories and functors by using the functorial perspective on clustering algorithms \cite{culbertson2014categorical, carlsson2013classifying, shiebler2020clustering}. 
\begin{definition}\label{definition:metcategories}
In the category $\met$ objects are metric spaces and the morphisms between $(X, d_X)$ and $(Y, d_Y)$ are non-expansive maps, or functions $f: X \rightarrow Y$ such that:
\begin{gather*}
    d_{Y}(f(x_1),f(x_2)) \leq d_X(x_1,x_2).
\end{gather*}
\end{definition}

\begin{definition}\label{definition:prt}
The objects in the category $\prt$ are tuples $(X, \px)$ where $\px$ is a partition of the set $X$. The morphisms in $\prt$ between $(X, \px)$ and $(Y, \py)$ are refinement-preserving maps, that is functions $f: X \rightarrow Y$ such that for any $S_X \in \px$, there exists some $S_Y \in \py$ with $f(S_X) \subseteq S_Y$.
\end{definition}

\begin{definition}\label{definition:nonoverlapping-flat-clustering-functor}
Given a subcategory $\db$ of $\met$, a  $\db$-clustering functor is a functor $F: \db \rightarrow \prt$ that is the identity on morphisms and underlying sets.
\end{definition}

That is, a  $\db$-clustering functor commutes with the forgetful functors from $\db$ and $\prt$ into $\set$. An example clustering functor is the $\delta$-single linkage clustering functor.
\begin{definition}\label{definition:single-linkage-flat}
The $\delta$-single linkage functor maps a metric space $(X,d_X)$ to the partition of $X$ in which the points $x_1,x_n$ are in the same partition if and only if there exists some sequence of points:
\begin{gather*}
    x_1, x_2, ..., x_{n-1}, x_n
\end{gather*}
such that for all $(x_i,x_{i+1})$ in this sequence we have:
\begin{gather*}
    d_X(x_i, x_{i+1}) \leq \delta
\end{gather*}
\end{definition}
In this section we will work with the restrictions of $\met$ and $\prt$ to the preorder subcategories in which morphisms are limited to inclusion maps $\iota(x) = x$. 
\begin{definition}\label{definition:metid}
$\metid$ is the subcategory of $\met$ in which the morphisms between $(X, d_X)$ and $(Y, d_Y)$ are limited to inclusion functions $\iota(x) = x$. $\metid$ is a preorder and we write:
\begin{gather*}
    (X, d_X) \leq_{\metid} (Y, d_Y)
\end{gather*}
to indicate that $X \subseteq Y$ and that $\iota: (X, d_X) \rightarrow (Y, d_Y)$ is non-expansive.
\end{definition}
Similarly, if $\db$ is a subcategory of $\metid$ then we write %
\begin{gather*}
    (X, d_X) \leq_{\db} (Y, d_Y)
\end{gather*}
to indicate that the inclusion map $\iota: (X, d_X) \rightarrow (Y, d_Y)$ is a morphism in $\db$. 
\begin{definition}\label{definition:prtid}
$\prtid$ is the subcategory of $\prt$ in which the morphisms between $(X, \px)$ and $(Y, \py)$ are limited to inclusion functions $\iota(x) = x$.
$\prtid$ is a preorder and we write:
\begin{gather*}
    (X, \px) \leq_{\prtid} (Y, \py)
\end{gather*}
to indicate that $X \subseteq Y$ and that $\iota: (X, \px) \rightarrow (Y, \py)$ is refinement-preserving.
\end{definition}
%

We can now frame our objective in terms of clustering functors. Suppose $\iota: \prtid \hookrightarrow \prt$ is the inclusion functor. Then given a subcategory $\db \subseteq  \metid$ (Definition \ref{definition:metid}), a discrete subcategory $\tb \subseteq \db$, and a functor $K: \tb \rightarrow \prtid$ such that:
\begin{gather*}
    \iota \circ K: \tb \rightarrow \prt
\end{gather*}
is a $\tb$-clustering functor (Definition \ref{definition:nonoverlapping-flat-clustering-functor}), find the best functor $F: \db \rightarrow \prtid$ such that:
\begin{gather*}
    \iota \circ F: \db \rightarrow \prt
\end{gather*}
is a $\db$-clustering functor and $F \circ G = K$ where $G:\tb \hookrightarrow \db$ is the inclusion functor.
\begin{center}
    \begin{tikzcd}[column sep=1in,row sep=1in]
\db (\subseteq \metid) \arrow{dr}{F} & \\
\tb (\subseteq \db) \arrow[hookrightarrow]{u}{G} \arrow{r}{K} & \prtid
\end{tikzcd}
\end{center}
Intuitively, $Ob(\tb)$ is the set of unlabelled training samples, $K$ defines the labels on these training samples, and $Ob(\db)$ is the set of testing samples.

We would like to use the Kan extensions of $K$ along $G$ to find this best clustering functor. However, these Kan extensions are not guaranteed to be clustering functors. 

For example, consider the case in which $\tb$ is the discrete category that contains the single-element metric space as its only object and $\db$ is the discrete category that contains two objects: the single-element metric space and $\rl$ equipped with the Euclidean distance metric \footnote{This counterexample due to Sam Staton}.
\begin{center}
\begin{tikzcd}[column sep=1in,row sep=1in]
\{(\{*\}, d_{\{*\}}), (\rl, d_{\rl})\} \arrow{dr}{F} & \\
\{(\{*\}, d_{\{*\}})\} \arrow[hookrightarrow]{u}{G} \arrow{r}{K} & \prtid
\end{tikzcd}
\end{center}
Since $\db$ is a discrete category, the behavior of $K$ on $(\{*\}, d_{\{*\}})$ will not affect the behavior of the left and right Kan extensions of $K$ along $G$ on $(\rl, d_{\rl})$. The left Kan extension of $K$ along $G$ will always map $(\rl, d_{\rl})$ to the initial object of $\prtid$ (the empty set). That is, the left Kan extension does not satisfy the conditions of Definition \ref{definition:nonoverlapping-flat-clustering-functor} since it does not act as the identity on the underlying set $\rl$. Futhermore, the right Kan extension of $K$ along $G$ will not exist because $\prtid$ does not have a terminal object. 

In order to solve this problem with Kan extensions we need to add a bit more structure. Suppose $Ob(\db)$ is the discrete category with the same objects as $\db$ and define the following:
\begin{definition}\label{definition:kl-functor}
The functor $K_L: Ob(\db) \rightarrow \prtid$ is equal to $K$ on $\tb$ and maps each object $(X, d_X)$ in $Ob(\db) - Ob(\tb)$ to $\xrefined$. 
\end{definition}
\begin{definition}\label{definition:kr-functor}
The functor $K_R: Ob(\db) \rightarrow \prtid$ is equal to $K$ on $\tb$ and maps each object $(X, d_X)$ in $Ob(\db) - Ob(\tb)$ to $(X, \{X\})$.

\end{definition}
Intuitively, $K_L$ and $K_R$ are extensions of $K$ to all of the objects in $\db$. For any metric space $(X, d_X)$ not in $Ob(\tb)$ the functor $K_L$ maps $(X, d_X)$ to the finest possible partition of $X$ and $K_R$ maps $(X, d_X)$ to the coarsest possible partition of $X$.

%
%
Suppose we go back to the previous example in which $\tb$ is the discrete category containing only the single-element metric space and $\db$ is the discrete category containing both the single-element metric space and $(\rl, d_{\rl})$. Since:
\begin{gather*}
K_L(\rl, d_{\rl}) = (\rl, \{\{x\}\ |\ x \in \rl\})
\end{gather*}
the left Kan extension of $K_L$ along the inclusion
\begin{gather*}
    G: Ob(\db) \hookrightarrow \db
\end{gather*}
must map $(\rl, d_{\rl})$ to the $\leq_{\prtid}$-smallest $(X, \px)$ such that:
\begin{gather*}
(\rl, \{\{x\}\ |\ x \in \rl\}) \leq_{\prtid}
(X, \px)
\end{gather*}
which is $(X, \px) = (\rl, \{\{x\}\ |\ x \in \rl\})$. Similarly, since:
\begin{gather*}
K_R(\rl, d_{\rl}) = (\rl, \{\rl\})
\end{gather*}
the right Kan extension of $K_R$ along the inclusion
\begin{gather*}
    G: Ob(\db) \hookrightarrow \db
\end{gather*}
must map $(\rl, d_{\rl})$ to the $\leq_{\prtid}$-largest $(X, \px)$ such that:
\begin{gather*}
(X, \px) \leq_{\prtid} (\rl, \{\rl\})
\end{gather*}
which is $(X, \px) = (\rl, \{\rl\})$. We can apply the same logic to the behavior of the Kan extensions on the single-element metric space as well, so the composition of
\begin{gather*}
    \iota: \prtid \hookrightarrow \prt
\end{gather*}
to either Kan extension yields a $\db$-clustering functor.

We can now build on this perspective to construct our optimal clustering functor extensions of $K$. 

\begin{proposition}\label{proposition:maximally-refined}
Consider the map $\lgkl: \db \rightarrow \prtid$ that acts as the identity on morphisms and sends the metric space $(X, d_X)$ to the partition of $X$ defined by the transitive closure of the relation $R$ where for $x_1, x_2 \in X$ we have $x_1 R x_2$ if and only if there exists some metric space $(X', d_{X'}) \in \tb$ where:
\begin{gather*}
     (X', d_{X'}) \leq_{\db} (X, d_{X})
\end{gather*}
and $x_1, x_2$ are in the same cluster in $K(X', d_{X'})$. The map:
\begin{gather*}
    \iota \circ \lgkl: \db \rightarrow \prt
\end{gather*}
is a $\db$-clustering functor. (Proof in Section \ref{proof:maximally-refined})
\end{proposition}

\begin{proposition}\label{proposition:maximally-coarse}
Consider the map $\rgkr: \db \rightarrow \prtid$ that acts as the identity on morphisms and sends the metric space $(X, d_X)$ to the partition of $X$ defined by the transitive closure of the relation $R$ where for $x_1, x_2 \in X$ we have $x_1 R x_2$ if and only if there exists no metric space $(X', d_{X'}) \in \tb$ where:
\begin{gather*}
     (X, d_{X}) \leq_{\db} (X', d_{X'})
\end{gather*}
and $x_1, x_2$ are in different clusters in $K(X', d_{X'})$. The map:
\begin{gather*}
    \iota \circ \rgkr: \db \rightarrow \prt
\end{gather*}
is a $\db$-clustering functor. (Proof in Section \ref{proof:maximally-coarse})
%
\end{proposition}

We can also make the following claim:
\begin{proposition}\label{equaltok}
Suppose there exists some functor $F: \db \rightarrow \prtid$ such that
\begin{gather*}
    \iota \circ F: \db \rightarrow \prt
\end{gather*}
is a $\db$-clustering functor and $F \circ G = K$. Then for $(X,d_X) \in \tb$ we have that:
\begin{align*}
    F(X, d_X) = \\
    K(X, d_X) = \\
    \lgkl(X, d_X) = \\
    \rgkr(X, d_X) 
\end{align*}
(Proof in Section \ref{proof:equaltok})
\end{proposition}

We can now put everything together and construct the functors $\lgkl,\rgkr$ as Kan extensions.
\begin{proposition}\label{clusteringproposition}
%
Suppose there exists some functor $F: \db \rightarrow \prtid$ such that
\begin{gather*}
    \iota \circ F: \db \rightarrow \prt
\end{gather*}
is a $\db$-clustering functor and $F \circ G = K$. 

Then $\lgkl: \db \rightarrow \prtid$ (Proposition \ref{proposition:maximally-refined}) is the left Kan extension of $K_L: Ob(\db) \rightarrow \prtid$ along the inclusion functor $G: Ob(\db) \hookrightarrow \db$.
\begin{center}
\begin{tikzcd}[column sep=1in,row sep=1in]
\db (\subseteq \metid) \arrow{dr}{\lgkl} & \\
Ob(\db) \arrow[hookrightarrow]{u}{G} \arrow{r}{K_L} & \prtid
\end{tikzcd}
\end{center}

In addition $\rgkr: \db \rightarrow \prtid$ (Proposition \ref{proposition:maximally-coarse}) is the right Kan extension of $K_R: Ob(\db) \rightarrow \prtid$ along the inclusion functor $G: Ob(\db) \hookrightarrow \db$.
\begin{center}
    \begin{tikzcd}[column sep=1in,row sep=1in]
\db (\subseteq \metid) \arrow{dr}{\rgkr} & \\
Ob(\db) \arrow[hookrightarrow]{u}{G} \arrow{r}{K_R} & \prtid
\end{tikzcd}
\end{center}
(Proof in Section \ref{proof:clusteringproposition})
\end{proposition}

Note that when $Ob(\tb) = \emptyset$ we have for any $(X, d_X) \in Ob(\db)$ that:
\begin{align*}
    &\lgkl(X, d_X) = K_L(X, d_X) = (X, \{\{x\} \ |\ x \in X\})
    \\
    &\rgkr(X, d_X) = K_R(X, d_X) = (X, \{\{X\}\})
\end{align*}
In general for any metric space $(X, d_X) \in Ob(\db) - Ob(\tb)$ the functors $\lgkl, \rgkr$ respectively map $(X, d_X)$ to the finest (most clusters) and coarsest (fewest clusters) partitions of $X$ such that for any metric space $(X', d_{X'}) \in \tb$ we have:
\begin{gather*}
K(X', d_{X'})=\lgkl(X', d_{X'})=\rgkr(X', d_{X'})    
\end{gather*}
and $\lgkl, \rgkr$ are functors. For example, suppose we have a metric space $(X, d_X)$ where $X=\{x_1, x_2, x_3\}$. We can form the subcategories $\tb \subseteq \db \subseteq \metid$ where:
\begin{align*}
    & Ob(\tb)=\{(\{x_1,x_2\}, d_X), (\{x_1,x_3\}, d_X),  (\{x_2,x_3\}, d_X)\} 
    \\
    & Ob(\db) = Ob(\tb) \cup (\{x_1,x_2,x_3\}, d_X)
\end{align*}
$\tb$ is a discrete category and the only non-identity morphisms in $\db$ are the inclusions $\{x_i, x_j\} \hookrightarrow \{x_1, x_2, x_3\}$. Now define $K: \tb \rightarrow \prtid$ to be the following functor:
\begin{align*}
&K(\{x_1,x_2\}, d_X) = \{\{x_1, x_2\}\} 
\\
&K(\{x_1,x_3\}, d_X) = \{\{x_1\},\{x_3\}\} 
\\
&K(\{x_2,x_3\}, d_X) = \{\{x_2\},\{x_3\}\}
\end{align*}
In this case we have that:
\begin{align*}
    &K_L(\{x_1,x_2,x_3\}, d_X) = \{\{x_1\},\{x_2\},\{x_3\}\}
    \\
    &K_R(\{x_1,x_2,x_3\}, d_X) = \{\{x_1,x_2,x_3\}\}
\end{align*}
Since the only points that need to be put together are $x_1, x_2$ and there are no non-identity morphisms out of $\{x_1,x_2,x_3\}$ in $\db$, we have:
\begin{align*}
    &\lgkl(\{x_1,x_2,x_3\}, d_X) = \{\{x_1,x_2\}, \{x_3\}\}
    \\
    &\rgkr(\{x_1,x_2,x_3\}, d_X) = \{\{x_1,x_2,x_3\}\}
\end{align*}
As another example, suppose $\db$ is $\metid$ and $\tb$ is the discrete subcategory of $\db$ whose objects are all metric spaces with $\leq 2$ elements. Define the following $\tb$-clustering functor:
\begin{gather*}
    K(\{x_1, x_2\}, d) =
    \begin{cases}
    \{\{x_1, x_2\}\}& d(x_1, x_2) \leq \delta
    \\
    \{\{x_1\},\{x_2\}\}  & \text{else}
    \end{cases}
\end{gather*}
%
%

Now for some metric space $(X, d_X)$ with $|X| > 2$ and points $x_1, x_2 \in X$ we have that $\lgkl$ maps $x_1,x_2$ to the same cluster if and only if there exists some chain of points $x_1, \cdots, x_2$ in $X$ where for each pair of adjacent points $x'_1, x'_2$ in this chain and any metric space $(\{x'_1, x'_2\}, d_{X'}) \in \db$ equipped with a non-expansive inclusion map:
\begin{gather*}
\iota: (\{x'_1, x'_2\}, d_{X'})
\hookrightarrow
(X, d_X)
\end{gather*}
in $\db$, it must be that the points $x'_1, x'_2$ are in the same cluster in $K(\{x'_1, x'_2\}, d_{X'})$. This is the case if and only if:
\begin{gather*}
d_{X}(x'_1, x'_2) \leq \delta
\end{gather*}.
Therefore, $\lgkl$ maps $x_1,x_2$ to the same cluster if and only if $x_1,x_2$ are in the same connected component of the $\delta$-Vietoris Rips complex of $(X, d_X)$. $\lgkl$ is therefore the $\delta$-single linkage functor (Definition \ref{definition:single-linkage-flat}).

In contrast, since $|X| > 2$ there are no morphisms in $\db$ from $(X, d_X)$ to any metric spaces in $\tb$. Therefore:
\begin{gather*}
    \rgkr(X, d_X) = (X, \{X\})
\end{gather*}

\begin{figure}[h]
\centering
\includegraphics[width=7.5cm,height=7.5cm]{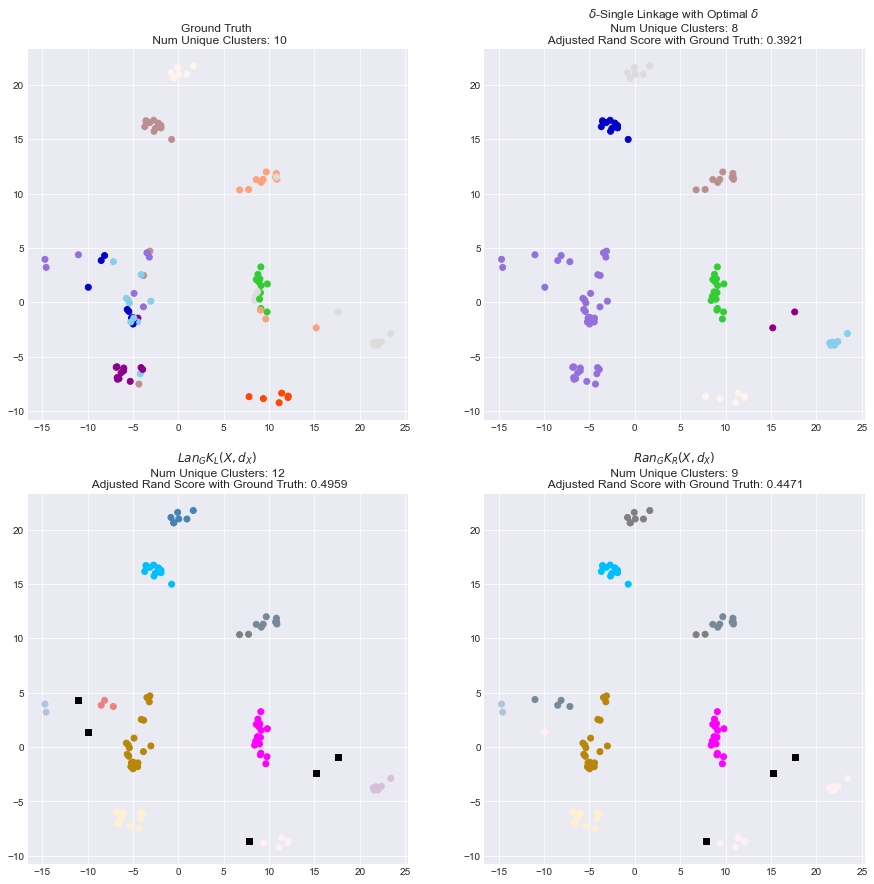}
\caption{
Cluster assignments of a $100$ point testing set $X_{te}$ from the Fashion MNIST dataset \citep{xiao2017/online} shown in UMAP space \citep{mcinnes2018umap}. Each color corresponds to a unique cluster, and points without clusters are shown as black squares. We show ground truth clothing categories, unsupervised $\delta$-single linkage cluster assignments ($\delta$ chosen via line search), and the $\lgkl, \rgkr$ supervised cluster assignments. The $\lgkl,\rgkr$ algorithms are trained on a separate $10,000$ point random sample $X_{tr}$ from the Fashion MNIST dataset. 
}
\label{fashionmnistclustering}
\end{figure}

We can use this strategy to learn a clustering algorithm from real-world data. Recall that the Fashion MNIST dataset \citep{xiao2017/online} contains images of clothing and the categories that each image falls into. Suppose that we have two subsets of this dataset: a training set $X_{tr}$ in which images are grouped by category and a testing set $X_{te}$ of ungrouped images. We can use UMAP \citep{mcinnes2018umap} to construct metric spaces $(X_{tr}, d_{X_{tr}})$ and $(X_{te}, d_{X_{te}})$ from these sets. 

Now suppose we would like to group the images in $X_{te}$ as similarly as possible to the grouping of the images in $X_{tr}$. 

For any category $\db^{*} \subseteq \met$ with:
\begin{gather*}
Ob(\db^{*}) = \{(X_{tr}, d_{X_{tr}}), (X_{te}, d_{X_{te}})\}
\end{gather*}
we can define subcategories $\tb \subseteq \db \subseteq \metid$ and functor $K: \tb \rightarrow \prtid$ as follows:
\begin{enumerate}
    \item Initialize $\tb$ to an empty category and $\db$ to be the discrete category with a single object $\{(X_{te}, d_{X_{te}})\}$.
    \item For every morphism
    \begin{gather*}
        f: (X_{tr}, d_{X_{tr}}) \rightarrow (X_{te}, d_{X_{te}})
    \end{gather*}
    in $\db^{*}$ and pair $(x_1,x_2) \in X_{tr}$ of samples in the same clothing category, add the object $(\{f(x_1), f(x_2)\}, d_{X_{te}})$ to $\tb$ and $\db$, add the inclusion morphism:
    \begin{gather*}
    \iota: (\{f(x_1), f(x_2)\}, d_{X_{te}}) \hookrightarrow (X_{te}, d_{X_{te}})
    \end{gather*}
    to $\db$, and define $K(\{f(x_1), f(x_2)\}, d_{X_{te}})$ to map $f(x_1)$ and $f(x_2)$ to the same cluster.
    %
    \item For every morphism
    \begin{gather*}
    f: (X_{te}, d_{X_{te}}) \rightarrow (X_{tr}, d_{X_{tr}})
    \end{gather*}
    in $\db^{*}$ define a metric space $(X'_{te}, d_{X'_{te}})$ where $X_{te} = X'_{te}$ and $d_{X_{te}} = d_{X'_{te}}$. Add the object $(X'_{te}, d_{X'_{te}})$ to $\tb$ and $\db$, add the inclusion map
    \begin{gather*}
        \iota: (X_{te}, d_{X_{te}}) \hookrightarrow (X'_{te}, d_{X'_{te}})
    \end{gather*}
    to $\db$ and define $K(X'_{te}, d_{X'_{te}})$ to be the partition of $X'_{te}$ defined by the preimages of the function $(h \circ f)$ where $h$ maps each element of $X_{tr}$ to the category of clothing it belongs to.
\end{enumerate}
We can now use $\lgkl$ and $\rgkr$ to partition $X_{te}$.
%
%
In Figure \ref{fashionmnistclustering} we
compare the clusterings produced by $\lgkl$ and $\rgkr$ to the ground truth clothing categories. We can compare clustering performance with the following metric:
\begin{definition}\label{definition:rand-score}
The Rand score \citep{rand1971objective, scikit-learn} between the partitions $\px,\px'$ of the set $X$ is the ratio:
\begin{gather*}
    RI(\px, \px') =
    \frac{
    |both(\px, \px')| +
    |neither(\px, \px')|}{|X|^2}
\end{gather*}
where:
\begin{gather*}
    both(\px, \px') = \\ 
    \{x_i, x_j \ |\ 
    \exists s_X \in \px, x_i, x_j \in s_X
    \wedge 
    \exists s'_X \in \px', x_i, x_j \in s'_X\}
    \\
    neither(\px, \px') = \\
    \{x_i, x_j \ |\
    \not\exists s_X \in \px, x_i, x_j \in s_X
    \wedge 
    \not\exists s'_X \in \px', x_i, x_j \in s'_X\}
\end{gather*}
\end{definition}
The value of the Rand score is heavily dependent on the number of clusters, which can make it difficult to interpret. Therefore, in practice we usually work with a chance-adjusted variant of the Rand score that is close to $0$ for a random partition and is exactly $1$ for identical partitions.
\begin{definition}\label{definition:adjusted-rand-score}
Suppose $\mathbf{P}^2_X$ is the set of all pairs of partitions of the set $X$ and $\mu_{\mathbf{P}^2_X}$ is the uniform distribution over $\mathbf{P}^2_X$. Then the adjusted Rand score \citep{hubert1985comparing, scikit-learn} between the partitions $\px,\px'$ of the set $X$ is the ratio:
\begin{gather*}
ARI(\px, \px') =
\frac{
    RI(\px, \px') - E_{\mu_{\mathbf{P}^2_X}}[RI]
}{
    \max_{\mathbf{P}^2_X}(RI) - E_{\mu_{\mathbf{P}^2_X}}[RI]
}
\end{gather*}
where $RI(\px, \px')$ is the Rand score between the partitions $\px,\px'$ (Definition \ref{definition:rand-score}).
\end{definition}

As a baseline we compute the $\delta$-single linkage clustering algorithm (Definition \ref{definition:single-linkage-flat}) with $\delta$ chosen via line search to maximize the adjusted Rand score (Definition \ref{definition:adjusted-rand-score}) with the ground truth labels.
As expected, we see that $\lgkl$ produces a finer clustering (more clusters) than does $\rgkr$ and that the clusterings produced by $\lgkl$ and $\rgkr$ are better than the clustering produced by single linkage in the sense of adjusted Rand score with ground truth.

\section{Meta-Supervised Learning}\label{section:kan-meta-learning}

%
%


Suppose $I$ is a set and $O$ is a partial order. A supervised learning algorithm maps a labeled dataset (set of pairs of points in $I \times O$) to a function $f: I \rightarrow O$. For example, both $\lgk$ and $\rgk$ from Section \ref{section:kan-classification} are supervised learning algorithms.

In this Section we use Kan extensions to derive supervised learning algorithms from pairs of datasets and functions. Our construction combines elements of Section \ref{section:kan-classification}'s point-level algorithms and Section \ref{section:kan-clustering}'s dataset-level functoriality constraints.

Suppose we have a finite partial order $S_f \subseteq (I \rightarrow O)$ of functions where for $f,f' \in S_f$ we have $f \leq f'$ when $\forall x \in I, f(x) \leq f'(x)$. 
\begin{proposition}\label{proposition:upper-antichain}
For any subset $S^{*}_f \subseteq S_f$ the upper antichain of $S^{*}_f$ is the set:
\begin{gather*}
\{f \ |\
    f \in S^{*}_f,\
    \not\exists f^{*} \in S^{*}_f, f < f^{*}\}
\}
\end{gather*}
The upper antichain of $S^{*}_f$ is an antichain in $S^{*}_f$, and for any function $f \in S^{*}_f$ there exists some function $f^{*}$ in the upper antichain of $S^{*}_f$ such that $f \leq f^{*}$.
(Proof in Section \ref{proof:upper-antichain})
\end{proposition}

Intuitively the upper antichain of $S^{*}_f$ is the collection of all functions $f\in S^{*}_f$ that are not strictly upper bounded by any other function in $S^{*}_f$. The upper antichain of an empty set is of course itself an empty set. 
\begin{definition}
%
We can form the following categories:
\begin{enumerate}
    %
    %
    \item[$\dcb$]: The objects in $\dcb$ are $\leq$-antichains of functions $X_f \subseteq S_f$. $\dcb$ is a preorder in which $X_{f} \leq X'_{f}$ if for $f \in X_f$ there must exist some $f' \in X'_f$ where $f \leq f'$. 
    %
    %
    \item[$\dbb$]: The objects in $\dbb$ are labeled datasets, or sets of pairs $U = \{(x,y)\ |\ x \in I, y\in O\}$.  $\dbb$ is a preorder such that $U \leq U'$ when for all $(x, y') \in U'$ there exists $(x, y) \in U$ where $y \leq y'$.
    \item[$\dab$]: A subcategory of $\dbb$ such that if $U \leq U' \in \dbb$ then $U \leq U' \in \dab$.
\end{enumerate} 
\end{definition}

\begin{proposition}\label{proposition:dab-dbb-dcb-categories}
$\dbb$ and $\dcb$ are preorder catgories. (Proof in Section \ref{proof:dab-dbb-dcb-categories})
\end{proposition}

Intuitively, $\dab$ is a collection of labeled training datasets and $\dbb$ is a collection of labeled testing datasets. We can define a functor that maps each training dataset to all of the trained models that agree with that dataset.
\begin{proposition}\label{proposition:meta-learn-k}
The map $K: \dab\rightarrow\dcb$ that acts as the identity on morphisms and maps the object $U \in \dab$ to the upper antichain of the following set:
\begin{align*}
    &S_{K}(U) =  \{f \ |\
        f \in S_f, \forall (x, y) \in U, f(x) \leq y\}
\end{align*}
is a functor. (Proof in Section \ref{proof:meta-learn-k})
\end{proposition}

Now define $G: \dab \hookrightarrow \dbb$ to be the inclusion functor. A functor $F: \dbb \rightarrow \dcb$ such that $F \circ G$ commutes with $K$ will then be a mapping from the testing datasets in $\dbb$ to collections of trained models.
%
\begin{center}
\begin{tikzcd}[column sep=1in,row sep=1in]
\dbb  \arrow{dr}{F} & \\
\dab \arrow[hookrightarrow]{u}{G} \arrow{r}{K} & \dcb
\end{tikzcd}
\end{center}
We can take the left and right Kan extensions of $K$ along the inclusion functor $G: \dab \hookrightarrow \dbb$ to find the optimal such mapping.

%
\begin{proposition}\label{factorizedfunctionskanproposition}
%
%
The map $\lgk$ that acts as the identity on morphisms and maps the object $U \in \dbb$ to the upper antichain of the following set:
\begin{gather*}
S_L(U) =
\bigcup_{\{U' \ |\ U'\in \dab,  U' \leq U\}}\  K(U')
\end{gather*}
is the left Kan extension of $K$ along $G$.

Next, the map $\rgk$ that acts as the identity on morphisms and maps the object $U \in \dbb$ to the upper antichain of the following set:
\begin{gather*}
S_R(U) =\\
 \{f \ |\
    f \in S_f,\
    \forall U' \in \{U' \ |\ U'\in \dab,  U \leq U'\}, \\
    \exists f' \in K(U'),\ 
    f  \leq f'
\}
\end{gather*}
is the right Kan extension of $K$ along $G$.
(Proof in Section \ref{proof:factorizedfunctionskanproposition})
\end{proposition}

Intuitively the functions in $\rgk(U)$ and $\lgk(U)$ are as large as possible subject to constraints imposed by the selection of sets in $Ob(\dab)$. The functions in $\lgk(U)$ are subject to a membership constraint and grow smaller when we remove objects from $Ob(\dab)$. The functions in $\rgk(U)$ are subject to an upper boundedness-constraint and grow larger when we remove objects from $Ob(\dab)$. 

Consider the extreme case where $Ob(\dab) = \emptyset$. For any $U \in \dbb$ we have that:
\begin{align*}
    &S_L(U) =
    \bigcup_{\{U' \ |\ U'\in \emptyset, \cdots\}}\  K(U') = 
    \emptyset
    \\
    &S_R(U) = 
        \{f \ |\
        f \in S_f,\
        \forall U' \in \emptyset, \ 
        \cdots
    \} = 
    S_f
\end{align*}
so $\lgk(U)$ is empty and $\rgk(U)$ is the upper antichain of $S_f$.

Now consider the extreme case where $Ob(\dab) = Ob(\dbb)$. For any $U \in \dbb$ and $f \in K(U)$ the functoriality of $K$ implies that:
\begin{gather*}
    \forall U' \in \{U' \ |\ U'\in \dab,  U \leq U'\}, \ 
    \exists f' \in K(U'),\ 
    f \leq f'
\end{gather*}
and therefore $f \in S_R(U)$. This implies $K(U) \leq \rgk(U)$. Similarly, for any $f \in \lgk(U)$ it must be that:
\begin{gather*}
     \exists
     U' \in \dab,\ 
     U' \leq U,\ 
     f \in K(U')
\end{gather*}
which by the functoriality of $K$ implies that
\begin{gather*}
    \exists f^{*} \in K(U), f \leq f^{*}
\end{gather*}
and therefore $\lgk(U) \leq K(U)$. Therefore in this extreme case we have:
\begin{gather*}
    \rgk(U) = \lgk(U) = K(U)
\end{gather*}

Let's now consider a more concrete example. Suppose $I=\rlp^2, O=\{\false,\true\}$, and $S_f$ is the finite set of linear classifiers $l: \rlp^2 \rightarrow \{\false, \true\}$ that can be expressed as:
\begin{align*}
l_{a,b}(x_1,x_2) =
    \begin{cases}
        \true & x_2 \leq a*x_1 + b \\
        \false & \text{else}
    \end{cases}
\end{align*}
where $a,b$ are integers in $(-100, 100)$.
Intuitively:
\begin{itemize}
    \item
    The classifiers in $\lgk(U)$ are selected to be the classifiers that predict $\true$ as often as possible among the set of all classifiers that have no false positives on some $U' \in \dab$ where $U' \leq U$.
    \item
    The classifiers in $\rgk(U)$ are constructed to predict $\true$ as often as possible subject to a constraint imposed by the selection of sets in $\dab$. For every set $U' \in \dab$ where $U \leq U'$ it must be that each classifier in $\rgk(U)$ is upper bounded at each point in $I$ by some classifier in $S_f$ with no false positives on $U'$.
\end{itemize}
We can give a concrete example to demonstrate this. Suppose that:
\begin{gather*}
    Ob(\dab) = \{
        \{((1,3), \false)\},
        \{((4,4), \false)\},
        \\
        \qquad\{
        ((2,2), \false),
        ((1,3), \false),
        ((4,4), \false)\}
        \}
    \}
    \\
    Ob(\dbb) = Ob(\dab) \cup \{
    \{
        ((1,3), \false),
        ((4,4), \false)\}
    \}
\end{gather*}
We can visualize $\dbb$ as follows:
\begin{center}
\begin{tikzcd}[column sep=0.1in,row sep=0.3in]
& \{((4,4), \false)\}
\\
\{((1,3), \false)\}  & \{
        ((1,3), \false),
        ((4,4), \false)\}  \arrow{l}{\leq}  \arrow{u}{\leq} &
\\
 & \{
        ((2,2), \false),
        ((1,3), \false),
        ((4,4), \false)\}
        \arrow{u}{\leq} &
\end{tikzcd}
\end{center}
We can see the following:
\begin{itemize}
\item $l_{(1,1)} \in K(\{((1,3), \false)\})$ since:
\begin{gather*}
l_{(1,1)}(1,3) =
\begin{cases}
\true & 3 \leq 1*1 + 1 \\
\false & \text{else}
\end{cases} = \false
\end{gather*}
but we have that:
\begin{align*}
&l_{(1,2)}(1,3) = 
\begin{cases}
\true & 3 \leq 1*1 + 2 \\
\false & \text{else}
\end{cases} =  \true
\\
&l_{(2,1)}(1,3) = 
\begin{cases}
\true & 3 \leq 2*1 + 1 \\
\false & \text{else}
\end{cases} = \true
\end{align*}
\item $l_{(0,2)} \in K(\{((1,3), \false)\})$ since:
\begin{gather*}
l_{(0,2)}(1,3) =
\begin{cases}
\true & 3 \leq 0*1 + 2 \\
\false & \text{else}
\end{cases} = \false
\end{gather*}
but we have that:
\begin{align*}
&l_{(0,3)}(1,3) = 
\begin{cases}
\true & 3 \leq 0*1 + 3 \\
\false & \text{else}
\end{cases} =  \true
\\
&l_{(1,2)}(1,3) = 
\begin{cases}
\true & 3 \leq 1*1 + 2 \\
\false & \text{else}
\end{cases} = \true
\end{align*}
\item $l_{(0,3)} \in K(\{((4,4), \false)\})$ since:
\begin{gather*}
l_{(0,3)}(4,4) =
\begin{cases}
    \true & 4 \leq 0*4 + 3 \\
    \false & \text{else}
\end{cases} =
\false
\end{gather*}
but we have that:
\begin{align*}
&l_{(1,3)}(4,4) =
\begin{cases}
    \true & 4 \leq 1*4 + 3 \\
    \false & \text{else}
\end{cases} = \true
\\
&l_{(0,4)}(4,4) =
\begin{cases}
    \true & 4 \leq 0*4 + 4 \\
    \false & \text{else}
\end{cases} = \true
\end{align*}
\item $l_{(0,1)} \in K(\{
        ((2,2), \false),
        ((1,3), \false),
        ((4,4), \false)\})$ since:
\begin{gather*}
l_{(0,1)}(2,2) =
\begin{cases}
    \true & 2 \leq 0*2 + 1 \\
    \false & \text{else}
\end{cases} =
\false
\\
l_{(0,1)}(1,3) =
\begin{cases}
    \true & 3 \leq 0*1 + 1 \\
    \false & \text{else}
\end{cases} =
\false
\\
l_{(0,1)}(4,4) =
\begin{cases}
    \true & 4 \leq 0*4 + 1 \\
    \false & \text{else}
\end{cases} =
\false
\end{gather*}
but we have that:
\begin{align*}
&l_{(1,1)}(4,4) =
\begin{cases}
    \true & 4 \leq 1*4 + 1 \\
    \false & \text{else}
\end{cases} = \true
\\
&l_{(0,2)}(2,2) =
\begin{cases}
    \true & 2 \leq 0*2 + 2 \\
    \false & \text{else}
\end{cases} = \true
\end{align*}
\end{itemize}

By the definition of $\lgk$ we have that:
\begin{gather*}
\lgk(\{
    ((1,3), \false),
    ((4,4), \false)\})
\end{gather*}
must contain $l_{(0,1)}$ since we have that:
\begin{gather*}
l_{(0,1)} \in K(\{
    ((2,2), \false),
    ((1,3), \false),
    ((4,4), \false)\})
\end{gather*}
but:
\begin{gather*}
l_{(0,2)} \not\in K(\{
    ((2,2), \false),
    ((1,3), \false),
    ((4,4), \false)\})
\end{gather*}
and:
\begin{gather*}
l_{(1,1)} \not\in K(\{
    ((2,2), \false),
    ((1,3), \false),
    ((4,4), \false)\})
\end{gather*}
Similarly, by the definition of $\rgk$ we have that:
\begin{gather*}
\rgk(\{
    ((1,3), \false),
    ((4,4), \false)\})
\end{gather*}
%
must contain $l_{(0,2)}$ since we have that:
\begin{gather*}
l_{(0,2)} \leq l_{(0,3)}
\qquad
l_{(0,2)} \leq l_{(0,2)}
\end{gather*}
but that there is no $l_{(a,b)}$ such that $l_{(0,2)} < l_{(a,b)}$ that is in both $K(\{ ((4,4), \false)\})$ and $K(\{((1,3), \false)\})$ since:
\begin{gather*}
l_{(1,2)} \not\in K(\{((4,4), \false)\})
\qquad
l_{(0,3)} \not\in K(\{((1,3), \false)\})
\end{gather*}
\begin{figure}[h]
\centering
\includegraphics[width=9cm,height=9cm]{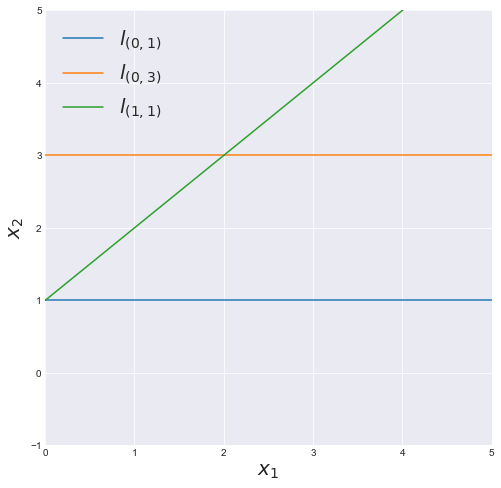}
\caption{
The decision boundaries defined by $l_{(1,1)}$, $l_{(0,3)}$, and $l_{(0,1)}$.
}
\label{linearfitfigure}
\end{figure}

\section{Function Approximation}\label{section:kan-function-approximation}
%
%
%

In many learning applications there may be multiple functions in a class that fit a particular set of data similarly well. In such a situation Occam's Razor suggests that we are best off choosing the simplest such function.
%
%
For example, we can choose the function with the smallest Kolmogorov complexity, also known as the minimum description length (MDL) function \citep{RISSANEN1978465}. In this Section we will explore how we can use Kan extensions to find the MDL function that fits a dataset.

Suppose $I$ is a set, $O$ is a partial order, and $S$ is a finite subset of $I$. We can define the following preorder:
\begin{definition}\label{definition:S-ordering}
Define the preorder $\lqs$ on $(I \rightarrow O)$ such that $f_1 \lqs f_2$ if and only if $\forall x \in S, f_1(x) \leq f_2(x)$. If $f_1 \lqs f_2, f_2 \not\leq_{S} f_1$ then write $f_1 <_{S} f_2$ and if $f_1 \lqs f_2 \lqs f_1$ then write $f_1 =_{S} f_2$. 
\end{definition}
%
%
%
Now suppose also that $\clq$ is some finite subset of the space of all functions $(I \rightarrow O)$ equipped with a total order $\leq_{c}$ such that $f_1 \leq_{c} f_2$ whenever the Kolmogorov complexity of $f_1$ is no larger than that of $f_2$. Note that functions with the same Kolmogorov complexity may be ordered arbitrarily in $\clq$.
\begin{proposition}\label{proposition:minimum-kolmogorov-subset}
Given a set of functions $S_{f} \subseteq \clq$ we can define a subset $S_{f_c} \subseteq S_{f}$, which we call the minimum Kolmogorov subset of $S_{f}$, such that for any function $f \in S_{f}$ the set $S_{f_c}$ contains exactly one function $f_c$ where $f =_{S} f_c$. This function $f_c$ satisfies $f_c \leq_{c} f$.
(Proof in Section \ref{proof:minimum-kolmogorov-subset})
\end{proposition}
We can use these constructions to define the following categories:
\begin{definition}\label{definition:approximation-categories}
Given the sets of functions $\sof \subseteq \stf \subseteq \clq$ define $\sofc$ to be the minimum Kolmogorov subset of $\sof$. We can construct the categories $\fab,\fbb,\fcb$ as follows.
\begin{itemize}
    \item The set of objects in the discrete category $\fab$ is $\sofc$.
    \item The set of objects in $\fbb$ is $\stf$. $\fbb$ is a preorder with morphisms $\lqs$.
    \item $\fcb$ is the subcategory of $\fbb$ in which objects are functions in $\sofc$ and morphisms are $\lqs$.
\end{itemize}
\end{definition}

Intuitively a functor $\fbb \rightarrow \fcb$ acts as a choice of a minimum Kolmogorov complexity function in $\sofc$ for each function in $\stf$. For example, if $\sof$ contains all linear functions and $\stf$ is the class of all polynomials then we can view a functor $\fbb \rightarrow \fcb$ as selecting a linear approximation for each polynomial in $\stf$.

\begin{proposition}\label{proposition:minimal-overapproximation}
For some function $g \in \stf$ define its minimal $S$-overapproximation to be the function $h \in \sofc$ where $g \leq_{S} h$ and $\forall h' \in \sofc$ where $g\leq_{S} h'$ we have $h \leq_{S} h'$. If this function exists it is unique.
\end{proposition}
\begin{proof}
Suppose $h_1,h_2$ are both minimal $S$-overapproximations of $g$. Then $h_1 \leq_{S} h_2$ and $h_2 \leq_{S} h_1$ which by the definition of $\sofc$ implies that $h_1=h_2$.
\end{proof}

\begin{proposition}\label{proposition:maximal-underapproximation}
For some function $g \in \stf$ define its maximal $S$-underapproximation to be the function $h \in \sofc$ where $h\leq_{S} g$ and $\forall h' \in \sofc$ where $h'\leq_{S} g$ we have $h' \leq_{S} h$. If this function exists it is unique.
\end{proposition}
\begin{proof}
Suppose $h_1,h_2$ are both maximal $S$-underapproximations of $g$. Then $h_2\leq_{S} h_1$ and $h_1 \leq_{S} h_2 $ which by the definition of $\sofc$ implies that $h_1=h_2$.
\end{proof}

\begin{proposition}\label{proposition:maximal-and-minimal}
Suppose that for some $g \in \stf$ there exists some $h \in \sofc$ such that $h(x) =_{S} g(x)$. Then $h$ will be both the minimal $S$-overapproximation and the maximal $S$-underapproximation of $g$.
\end{proposition}
\begin{proof}
To start, note that $h$ must satisfy $g \leq_{S} h$ and for any $h' \in \sofc$ we have $h =_{S} g\leq_{S} h'$ so $h$ is the minimal $S$-overapproximation of $g$. Next, note that $h$ must satisfy $h\leq_{S} g$ and for any $h' \in \sofc$ we have $h'\leq_{S} g =_{S} h$ so $h$ is also the maximal $S$-underapproximation of $g$.
\end{proof}

We can now show the following:
\begin{proposition}\label{approxfunctionkanproposition}
%
Define both $K: \fab \hookrightarrow \fcb$ and $G: \fab \hookrightarrow \fbb$ to be inclusion functors. Then:

\begin{itemize}
    \item Suppose that for any function $g \in \stf$ there exists a minimal $S$-overapproximation (Proposition \ref{proposition:minimal-overapproximation}) $h$ of $g$. Then the left Kan extension of $K$ along $G$ is the functor $\lgk$ that acts as the identity on morphisms and maps $g$ to $h$.
    \item Suppose that for any function $g \in \stf$ there exists a maximal $S$-underapproximation (Proposition \ref{proposition:maximal-underapproximation}) $h$ of $g$. Then the right Kan extension of $K$ along $G$ is the functor $\rgk$ that acts as the identity on morphisms and maps $g$ to $h$.
\end{itemize}
(Proof in Section \ref{proof:approxfunctionkanproposition})
\end{proposition}
\begin{center}
\begin{tikzcd}[column sep=1in,row sep=1in]
\fbb  \arrow{dr}{F} & \\
\fab \arrow[hookrightarrow]{u}{G} \arrow[hookrightarrow]{r}{K} & \fcb
\end{tikzcd}
\end{center}
Intuitively, the Kan extensions of the inclusion functor $K:\fab \rightarrow \fcb$ along the inclusion functor $G:\fab \rightarrow \fbb$ map a function $g \in \stf$ to its best $\sof$-approximations over the points in $S$. 

For example, suppose $I=O=\rl$, $g$ is a polynomial, $\sof$ is the set of lines defined by all pairs of points in $S$ and $\stf = \sof \cup g$. $\lgk$ and $\rgk$ may or may not exist depending on the choice of $S$ and $g$. In Figure \ref{linearfitfigure} we give an example $S,g$ in which $\lgk$ exists and $\rgk$ does not (left) and an example $S,g$ in which $\rgk$ exists and $\lgk$ does not (right).
\begin{figure}[h]
\centering
\includegraphics[width=9cm,height=5cm]{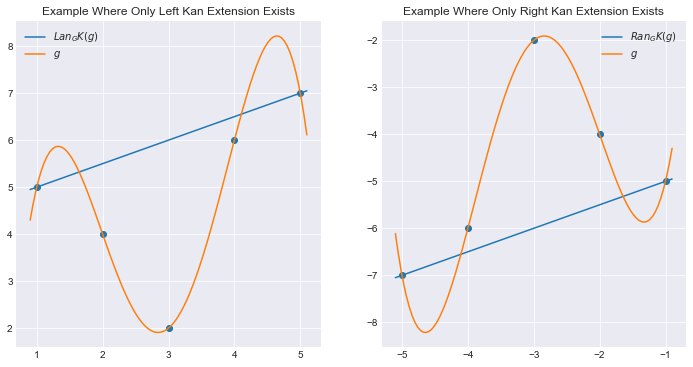}
\caption{
Left and right Kan extensions of $K: \fab \hookrightarrow \fcb$ along $G: \fab\hookrightarrow \fbb$ for two example sets $S$ and polynomials $g$ where $\sof$ is the class of lines and $\stf = \sof \cup g$.
}
\label{linearfitfigure}
\end{figure}

As another example, suppose $I=O=\rl$, $\sof$ is a subset of all polynomials of degree $|S|-1$ and $\stf$ is a subset of all functions $\rl\rightarrow\rl$. Since there always exists a unique $n-1$ degree polynomial through $n$ unique points, for any $S$ there exists some $\sof$ so that both $\lgk$ and $\rgk$ exist and map $g\in \stf$ to the unique $|S|-1$ degree polynomial that passes through the points $\{(x, g(x)) \ |\ x \in S\}$.

As another example, consider the classification case in which $I=\rl^a, O=\{\false,\true\}$, $\sof$ is a subset of all neural networks with a single hidden layer and $\stf$ is a subset of all functions $\rl^n \rightarrow \{\false,\true\}$. For any $S$ it is possible to select $\sof$ so that for any $g \in \stf$ the function $\lgk(g)$ is the smallest (minimum Kolmogorov complexity) single hidden layer neural network $f: \rl^a \rightarrow \{\false,\true\}$ in $\sof$ such that $\forall x \in S, f(x) = g(x)$.

\section{Discussion}

The category theoretic perspective on generalization that we introduce in this paper is fundamentally different from the traditional data science perspective. Intuitively, the traditional data science perspective is mean and percentile-focused whereas the category theoretic perspective is min and max-focused.
That is, traditional data science algorithms may have objectives like ``minimize total errors'' while the category theoretic algorithms we discuss in this work have objectives like ``minimize false positives subject to no false negatives on some set.'' As a result, algorithms built from the category theoretic perspective may behave more predictably, but can also be more sensitive to noise.




\bibliographystyle{ACM-Reference-Format}
\bibliography{generic}

\section{Appendix}
\subsection{Proof of Proposition \ref{supervisedpreorderproposition}}\label{proof:supervisedpreorderproposition}
\begin{proof}
We first need to show that $\lgk, \rgk$ are functors. For any $x_1 \leq x_2 \in \ib$ suppose that $\lgk(x_1) = \true$. Then $\exists x' \in \ib', x'  \leq x_1, K(x') = \true$. By transitivity we have $x' \leq x_2$, so:
\begin{align*}
\lgk(x_2) =\\
\begin{cases}
\true & \exists x' \in \ib', x' \leq x_2, K(x') = \true \\
\false  & \text{else}
\end{cases} = \\
\true
\end{align*}
and $\lgk$ is therefore a functor.

Next, for any $x_1 \leq x_2 \in \ib$ suppose that $\rgk(x_2) = \false$. Then $\exists x' \in \ib', x_2 \leq x', K(x') = \false$. By transitivity we have $x_1 \leq x'$, so:
\begin{align*}
\rgk(x_1) =\\
\begin{cases}
\false  & \exists x' \in \ib', x_1 \leq x', K(x') = \false \\
\true & \text{else}
\end{cases} =\\
\false
\end{align*}
and $\rgk$ is therefore a functor.

Next we will show that $\lgk$ is the left Kan extension of $K$ along $G$. If for some $x' \in \ib'$ we have that $K(x') = \true$ then:
\begin{align*}
\lgk(x') = 
\\
\begin{cases}
\true & \exists x'' \in \ib', x'' \leq x', K(x'') = \true \\
\false  & \text{else}
\end{cases} = 
\\
\true
\end{align*}
so we can conclude that $K \leq (\lgk \circ G)$. Now consider any other functor $M_L: \ib \rightarrow \{\false,\true\}$ such that $\forall x' \in \ib', K(x') \leq M_L(x')$. We must show that $\forall x \in \ib, \lgk(x) \leq M_L(x)$. For some $x \in \ib$ suppose $M_L(x)=\false$. Then since $M_L$ is a functor it must be that $\forall x' \in \ib', x' \leq x, M_L(x')=\false$. Since $K \leq (M_L \circ G)$ it must be that $\forall x' \in \ib', x' \leq x, K(x')=\false$. Therefore $\lgk(x)=\false$.

Next we will show that $\rgk$ is the right Kan extension of $K$ along $G$. If for some $x' \in \ib'$ we have that $K(x') = \false$ then:
\begin{align*}
\rgk(x') = \\
\begin{cases}
\false & \exists x'' \in \ib', x' \leq x'', K(x'') = \false \\
\true & \text{else}
\end{cases} =\\
\false
\end{align*}
so we can conclude that $(\rgk \circ G) \leq K$. Now consider any other functor $M_R: \ib \rightarrow \{\false,\true\}$ such that $\forall x' \in \ib', M_R(x') \leq K(x')$. We must show that $\forall x \in \ib, M_R(x) \leq \rgk(x)$. For some $x \in \ib$ suppose $M_R(x)=\true$. Then since $M_R$ is a functor it must be that $\forall x' \in \ib', x \leq x', M_R(x')=\true$. Since $(M_R \circ G) \leq K$ it must be that $\forall x' \in \ib', x \leq x', K(x')=\true$. Therefore $\rgk(x)=\true$.
\end{proof}


\subsection{Proof of Proposition \ref{optimizationproposition}}\label{proof:optimizationproposition}
\begin{proof}
First note that $l$ must be non-negative since each term can be expressed as $\max(0,\_)$. Next, suppose that $l(f) = 0$. Then it must be that for any $x_0,x_1 \in \ib'$ such that $K(x_0)=\false, K(x_1)=\true$ we have that $f(x_0) \leq f(x_1)$. As a result, for any $x\in \ib'$ there can only exist some $x' \in \ib'$ where $f(x) \leq f(x'), K(x') = \false$ when $K(x) = \false$. Similarly, there can only exist some $x' \in \ib'$ where $f(x') \leq f(x), K(x') = \true$ when $K(x) = \true$.  Therefore:
\begin{gather*}
    K(x)=(Lan_{f \circ G}K \circ f)(x)=(Ran_{f \circ G}K \circ f)(x)
\end{gather*}
\end{proof}

\subsection{Proof of Proposition \ref{proposition:maximally-refined}}\label{proof:maximally-refined}
\begin{proof}\label{proof:maximally-refined}
%
%
$\lgkl$ trivially acts as the identity on morphisms and underlying sets and preserves composition and identity so we simply need to show that when:
\begin{gather*}
    (X, d_X) \leq_{\db} (Y, d_{Y})
\end{gather*}
then
\begin{gather*}
    \lgkl(X, d_X) \leq_{\prtid} \lgkl(Y, d_{Y})
\end{gather*}
Suppose there exists some $x,x^{*} \in X$ in the same cluster in $\lgkl(X, d_X)$. Then by the definition of $\lgkl$ there must exist some sequence
\begin{gather*}
   (X_1, d_{X_1}), (X_2, d_{X_2}), \cdots, (X_n, d_{X_n}) \in \tb
\end{gather*}
where $x \in X_1, x^{*} \in X_n$ and each:
\begin{gather*}
    (X_i, d_{X_i}) \leq_{\db} (X, d_{X})
\end{gather*}
as well as some sequence
\begin{gather*}
    x_1, x_2, \cdots, x_{n-1}, \ \text{such that}\ 
    x_i \in X_i, x_i \in X_{i+1}
\end{gather*}
where the pair $(x, x_1)$ is in the same cluster in $K(X_1, d_{X_1})$, the pair $(x_{n-1}, x^{*})$ is in the same cluster in $K(X_n, d_{X_n})$, and for each $1 < i < n$ the pair $(x_{i-1}, x_{i})$ is in the same cluster in $K(X_{i}, d_{X_{i}})$.
Since it must be that each:
\begin{gather*}
    (X_i, d_{X_i}) \leq_{\db} (Y, d_{Y})
\end{gather*}
as well then by the definition of $\lgkl$ it must be that $x,x^{*}$ are in the same cluster in $\lgkl(Y, d_{Y})$.
\end{proof}


%
\subsection{Proof of Proposition \ref{proposition:maximally-coarse}}\label{proof:maximally-coarse}
\begin{proof}
$\rgkr$ trivially acts as the identity on morphisms and underlying sets and preserves composition and identity so we simply need to show that when:
\begin{gather*}
    (X, d_X) \leq_{\db} (Y, d_{Y})
\end{gather*}
then:
\begin{gather*}
    \rgkr(X, d_X) \leq_{\prtid} \rgkr(Y, d_{Y})
\end{gather*}

Suppose the points $x, x^{*} \in X$ are in the same cluster in $\rgkr(X, d_X)$. Then by the definition of $\rgkr$ there cannot be any $(X', d_{X'})$ in $\tb$ such that:
\begin{gather*}
    (X, d_X) \leq_{\db} (X', d_{X'})
\end{gather*}
and $x, x^{*}$ are in different clusters in $\rgkr(X', d_{X'})$. By transitivity this implies that there cannot be any $(X'', d_{X''})$ in $\tb$ such that:
\begin{gather*}
    (Y, d_Y) \leq_{\db} (X'', d_{X''})
\end{gather*}
and $x, x^{*}$ are in different clusters in $\rgkr(X'', d_{X''})$. By the definition of $\rgkr$ the points $x, x^{*}$ must therefore be in the same cluster in $\rgkr(Y, d_Y)$.
\end{proof}
%


%
\subsection{Proof of Proposition \ref{equaltok}}\label{proof:equaltok}
\begin{proof}
%
Since each of:
\begin{align*}
    &\iota \circ F: \db \rightarrow \prt \\
    &\iota \circ \rgkr: \db \rightarrow \prt \\
    &\iota \circ \lgkl: \db \rightarrow \prt
\end{align*}
are $\db$-clustering functors we simply need to prove that all three functors generate the same partition of $X$ for any input $(X, d_X) \in \tb$.

Consider some $(X, d_X) \in \tb$ and two points $x, x^{*} \in X$. Suppose $x, x^{*}$ are in different clusters in
\begin{gather*}
    K(X, d_X) = F(X, d_X)
\end{gather*}
. Then since $F$ is a $\db$-clustering functor it must be that for any sequence
\begin{gather*}
  (X_1, d_{X_1}), (X_2, d_{X_2}), \cdots, (X_n, d_{X_n}) \in \tb
\end{gather*}
where $x \in X_1, x^{*} \in X_n$ and each:
\begin{gather*}
    (X_i, d_{X_i}) \leq_{\db} (X, d_{X})
\end{gather*}
and any sequence
\begin{gather*}
    x_1, x_2, \cdots, x_{n-1}, \ \text{such that}\ 
    x_i \in X_i, x_i \in X_{i+1}
\end{gather*}
one of the following must be true:
\begin{itemize}
    \item The pair  $(x, x_1)$ are in different clusters in $F(X_1, d_{X_1})$
    \item The pair $(x_{n-1}, x^{*})$ are in different clusters in $F(X_n, d_{X_n})$
    \item For some $1 < i < n$ the pair $(x_{i-1}, x_{i})$ are in different clusters in $F(X_{i}, d_{X_{i}})$
\end{itemize}
This implies that in  $\lgkl(X, d_X)$ the points $x, x^{*}$ must be in different clusters. Similarly, since $(X, d_X) \leq_{\db} (X, d_X)$, by Proposition \ref{proposition:maximally-coarse} it must be that $x, x^{*}$ are in different clusters in $\rgkr(X, d_X)$.



Now suppose $x, x^{*}$ are in the same cluster in
\begin{gather*}
    K(X, d_X) = F(X, d_X)
\end{gather*}.
Since $(X, d_X) \leq_{\db} (X, d_X)$, by Proposition \ref{proposition:maximally-refined} it must be that $x,x^{*}$ are in the same cluster in $\lgkl(X, d_X)$. Similarly, since $F$ is a $\db$-clustering functor there cannot exist any metric space $(X', d_{X'}) \in \tb$ where:
\begin{gather*}
    (X, d_X) \leq_{\db} (X', d_{X'})
\end{gather*}
and $x,x^{*}$ are in different clusters in
\begin{gather*}
    K(X', d_{X'}) = F(X', d_{X'})
\end{gather*}.
Therefore $x,x^{*}$ are in the same cluster in $\rgkr(X, d_X)$.
\end{proof}


\subsection{Proof of Proposition \ref{clusteringproposition}}\label{proof:clusteringproposition}
\begin{proof}
%
%

To start, note that Proposition \ref{equaltok} implies that for any $(X, d_X) \in \tb$ we have:
\begin{align*}
    \lgkl(X, d_X) = K(X, d_X) = \rgkr(X, d_X)
\end{align*}
By the definition of $K_L, K_R$ we can therefore conclude that for any $(X, d_X) \in \db$ we have:
\begin{align*}
    & K_L(X, d_X) \leq_{\prtid} \lgkl(X, d_X) \\
    & \rgkr(X, d_X) \leq_{\prtid} K_R(X, d_X)
\end{align*}
Next, consider any functor $M_{L}: \db \rightarrow \prtid$ such that for all $(X, d_X) \in \db$ we have:
\begin{gather*}
    K_L(X, d_X) \leq_{\prtid} (M_L \circ G)(X, d_X)
\end{gather*}
We must show that for any $(X, d_X) \in \db$ we have:
\begin{gather*}
    \lgkl(X, d_X) \leq_{\prtid} M_L(X, d_X)
\end{gather*}
To start, note that for any $x, x^{*} \in X$ that are in the same cluster in $\lgkl(X, d_X)$ by the definition of $\lgkl$ there must exist some sequence:
\begin{gather*}
   (X_1, d_{X_1}), (X_2, d_{X_2}), \cdots, (X_n, d_{X_n}) \in \tb
\end{gather*}
where $x \in X_1, x^{*} \in X_n$ and each:
\begin{gather*}
    (X_i, d_{X_i}) \leq_{\db} (X, d_{X})
\end{gather*}
as well as some sequence
\begin{gather*}
    x_1, x_2, \cdots, x_{n-1}, \ \text{such that}\ 
    x_i \in X_i, x_i \in X_{i+1}
\end{gather*}
where the pair $(x, x_1)$ is in the same cluster in $K_L(X_1, d_{X_1})$, the pair $(x_{n-1}, x^{*})$ is in the same cluster in $K_L(X_n, d_{X_n})$, and for each $1 < i < n$ the pair $(x_{i-1}, x_{i})$ is in the same cluster in $K_L(X_{i}, d_{X_{i}})$.
Now since for each $(X_{i}, d_{X_{i}})$ in this sequence we have that:
\begin{gather*}
    K_L(X_{i}, d_{X_{i}}) \leq_{\prtid}  M_L(X_{i}, d_{X_{i}})
\end{gather*}
it must be that the pair $(x, x_1)$ is in the same cluster in $M_L(X_1, d_{X_1})$, the pair $(x_{n-1}, x^{*})$ is in the same cluster in $M_L(X_n, d_{X_n})$, and for each $1 < i < n$ the pair $(x_{i-1}, x_{i})$ is in the same cluster in $M_L(X_{i}, d_{X_{i}})$.

Since $M_L$ is a functor it must therefore be that the pair $x, x^{*}$ is in the same cluster in $M_L(X, d_X)$ and therefore:
\begin{gather*}
    \lgkl(X, d_X) \leq_{\prtid} M_L(X, d_X)
\end{gather*}.

Next, consider any functor $M_{R}: \db \rightarrow \prtid$ such that for all $(X, d_X)$ in $\db$:
\begin{gather*}
    (M_R \circ G)(X, d_X) \leq_{\prtid} K_R(X, d_X)
\end{gather*}
We must show that for any $(X, d_X)$ in $\db$ we have:
\begin{gather*}
    M_R(X, d_X) \leq_{\prtid} \rgkr(X, d_X)
\end{gather*}
To start, note that for any $x, x^{*} \in X$ such that $x, x^{*}$ are not in the same cluster in $\rgkr(X, d_X)$ by the definition of $\rgkr$ there must exist some:
\begin{gather*}
    (X', d_{X'}) \in Ob(\db),
    (X, d_{X}) \leq_{\db} (X', d_{X'})
\end{gather*}
where $x, x^{*}$ are not in the same cluster in $K_R(X', d_{X'})$. Now since:
\begin{gather*}
    M_R(X', d_{X'}) \leq_{\prtid} K_R(X', d_{X'})
\end{gather*}
it must be that $x, x^{*}$ are not in the same cluster in $M_R(X', d_{X'})$. Since $M_R$ is a functor we have:
\begin{gather*}
    M_R(X, d_{X}) \leq_{\prtid} M_R(X', d_{X'})
\end{gather*}
so $x, x^{*}$ are also not in the same cluster in $M_R(X, d_{X})$ and therefore:
\begin{gather*}
    M_R(X, d_X) \leq_{\prtid} \rgkr(X, d_X)
\end{gather*}.
\end{proof}


\subsection{Proof of Proposition \ref{proposition:upper-antichain}}\label{proof:upper-antichain}
\begin{proof}

Suppose $f_1,f_2$ are in the upper antichain of $S^{*}_f \subseteq S_f$ and $f_1 \leq f_2$. Then since 
\begin{gather*}
    \not\exists f^{*}_1 \in S^{*}_f, f_1 < f^{*}_1
\end{gather*}
it must be that $f_1 = f_2$ and we can conclude that the upper antichain is an antichain.

Next, for any function $f \in S^{*}_f$ consider the set $\{f^{*} \in S^{*}_f, f < f^{*}\}$. Since $S_f$ is finite this set must have finite size. If this set is empty then $f$ is in the upper antichain of $S^{*}_f$. If this set has size $n$ then for any $f^{*}$ in this set the set $\{f^{**} \in S^{*}_f, f^{*} < f^{**}\}$ must have size strictly smaller than $n$. We can therefore conclude by induction that the upper antichain of $S^{*}_f$ contains at least one function $f^{*}$ where $f \leq f^{*}$.
\end{proof}
%

\subsection{Proof of Proposition \ref{proposition:dab-dbb-dcb-categories}}\label{proof:dab-dbb-dcb-categories}
\begin{proof}

~\\$\dcb$ ~\\
We trivially have $X_f \leq X_f$ in $\dcb$. To see that $\leq$ is transitive in $\dcb$ simply note that if $X_{f_1} \leq X_{f_2}$ and $X_{f_2} \leq X_{f_3}$ then for $f_1 \in X_{f_1}$ there must exist $f_2 \in X_{f_2}, f_1 \leq f_2$, which implies that there must exist $f_3 \in X_{f_3}, f_1 \leq f_2 \leq f_3$.

~\\$\dbb$ ~\\
We trivially have $U\leq U$ in $\dbb$. To see that $\leq$ is transitive in $\dbb$ simply note that if $U_1 \leq U_2$ and $U_2 \leq U_3$ in $\dbb$ then for 
$(x, y_3) \in U_3$ there must exist $(x, y_2) \in U_2, y_2 \leq y_3$ which implies that there must exist $(x, y_1) \in U_1, y_1 \leq y_2 \leq y_3$.
\end{proof}
%

\subsection{Proof of Proposition \ref{proposition:meta-learn-k}}\label{proof:meta-learn-k}
\begin{proof}
%
%
To start, note that $K$ maps objects in $\dab$ to objects in $\dcb$ since the upper antichain of $S_K(U)$ must be an antichain in $S_f$ by Proposition \ref{proposition:upper-antichain}. 

Next, we need to show that if $U \leq U'$ then $K(U) \leq K(U')$. For any $x,y' \in U'$ it must be that there exists $(x, y) \in U$ where $y \leq y'$, so if $f \in K(U)$ then by the definition of $K$ we have $f(x) \leq y \leq y'$. Therefore $f \in S_K(U')$, so by Proposition \ref{proposition:upper-antichain} $K(U')$ contains $f'$ where $f \leq f'$. Therefore $K(U) \leq K(U')$.
\end{proof}


\subsection{Proof of Proposition \ref{factorizedfunctionskanproposition}}\label{proof:factorizedfunctionskanproposition}
\begin{proof}

We first need to show that $\lgk$ is a functor. Note that $\lgk$ maps objects in $\dbb$ to objects in $\dcb$ since the upper antichain of $S_L(U)$ must be an antichain in $S_f$.

Next, suppose $U_1 \leq U_2$ and that $f \in \lgk(U_1)$. Consider the set of all $U' \in \dab$ where $U' \leq U_1$. Since $U_1 \leq U_2$ this is a subset of the set of all $U' \in \dab$ where $U' \leq U_2$.
Since $S_L(U_1)$ is defined to be a union of the elements in the set we have that $S_L(U_1) \subseteq S_L(U_2)$. Since $f \in \lgk(U_1)$ implies that $f \in S_L(U_1)$ this implies that $f \in S_L(U_2)$ as well.
Proposition \ref{proposition:upper-antichain} then implies that there must exist $f' \in \lgk(U_2)$ where $f \leq f'$ and therefore $\lgk(U_1) \leq \lgk(U_2)$.

Next, we will show that $\lgk$ is the left Kan extension of $K$ along $G$.
\begin{itemize}
    \item Consider some $U \in \dab$ and $f \in K(U)$. Since $U \leq U$ we have by the definition of $S_L$ that $f \in S_L(U)$. Proposition \ref{proposition:upper-antichain} then implies that $\exists f' \in \lgk(U)$ such that $f \leq f'$. This implies that $K \leq \lgk \circ G$.
    \item Now consider any functor $M_{L}: \dbb \rightarrow \dcb$ such that $K \leq (M_L \circ G)$. We must show that $\lgk \leq M_L$. For some $U \in \dbb$ suppose $f \in \lgk(U)$. By the definition of $S_L$ there must exist some $U' \in \dab$ where $U' \leq U$ such that $f \in K(U')$. Since $K(U') \leq M_L(U')$ there must exist some $f' \in M_L(U')$ where $f \leq f'$. Since $M_L$ is a functor we have $M_L(U') \leq M_L(U)$ which implies that there must exist some $f^{*} \in M_R(U)$ where $f \leq f' \leq f^{*}$. Therefore $\lgk \leq M_L$.
\end{itemize}

Next, we need to show that $\rgk$ is a functor. Note that $\rgk$ maps objects in $\dbb$ to objects in $\dcb$ since the upper antichain of $S_R(U)$ must be an antichain in $S_f$. 
Next, suppose $U_1 \leq U_2$ and that $f \in \rgk(U_1)$. Consider the set of all $U' \in \dab$ where $U_2 \leq U'$. Since $U_1 \leq U_2$ this is a subset of the set of all $U' \in \dab$ where $U_1 \leq U'$.
Therefore by the definition of $S_R$ we have that $S_R(U_1) \subseteq S_R(U_2)$. Since $f \in \rgk(U_1)$ implies that $f \in S_R(U_1)$ this implies that $f \in S_R(U_2)$ as well. Proposition \ref{proposition:upper-antichain} then implies that there must exist $f' \in \rgk(U_2)$ where $f \leq f'$ and therefore $\rgk(U_1) \leq \rgk(U_2)$.

Next, we will show that $\rgk$ is the right Kan extension of $K$ along $G$.
\begin{itemize}
    \item For $U \in \dab$ since $U \leq U$ we have that when $f \in S_R(U)$ we have by the definition of $S_R$ that $\exists f' \in K(U)$ such that $f \leq f'$. Since $\rgk(U)$ is a subset of $S_R(U)$ this implies that $\rgk \circ G \leq K$.
    \item Now consider any functor $M_{R}: \dbb \rightarrow \dcb$ such that $(M_R \circ G) \leq K$. We must show that $M_R \leq \rgk$. For some $U \in \dbb$ suppose $f \in M_R(U)$. Since $M_R$ is a functor it must be that for all $U' \in \dab$ where $U \leq U'$ we have that $M_R(U) \leq M_R(U')$ and therefore $\exists f'_{M_R} \in M_R(U'), f \leq f'_{M_R}$. Since $(M_R \circ G) \leq K$ this implies that for all $U' \in \dab$ where $U \leq U'$ we have that $\exists f'_{K} \in K(U'), f \leq f'_{M_R} \leq f'_{K}$. By the definition of $S_R$ this implies that $f \in S_R(U)$. Proposition \ref{proposition:upper-antichain} therefore implies that there exists $f'_{R} \in \rgk(U)$ such that $f \leq f'_{R}$, and therefore $M_R(U) \leq \rgk(U)$. 
\end{itemize}
\end{proof}


\subsection{Proof of Proposition \ref{proposition:minimum-kolmogorov-subset}}
\begin{proof}\label{proof:minimum-kolmogorov-subset}
For any function $f \in S_f$ there must exist some $f_c = \min_{\leq_c} \{f' \ |\ f' \in S_{f}, f' =_{S} f \}$ since $\{f' \ |\ f' \in S_{f}, f' =_{S} f \}$ is a nonempty finite total $\leq_{c}$-order. Therefore we can define a map that sends each $f \in S_f$ to $f_c$. Define $S_{f_c}$ to be the image of this map. 

Since this map will send all $f \in S_f$ in the same $=_{S}$ equivalence class to the same function in that  $=_{S}$ equivalence class, $S_{f_c}$ contains exactly one function $f_c$ where $f =_{S} f_c$. This function $f_c$ satisfies $f_c \leq_{c} f$.
\end{proof}


%
\subsection{Proof of Proposition \ref{approxfunctionkanproposition}}\label{proof:approxfunctionkanproposition}
\begin{proof}

We first show that $\lgk$ is a functor when it exists. Since $\fbb,\fcb$ are preorders we simply need to show that when $f_1 \lqs f_2$ then $\lgk(f_1) \lqs \lgk(f_2)$. Since $f_2 \lqs \lgk(f_2)$ by the definition of the minimal $S$-overapproximation of $f_2$ we have that $f_1 \lqs \lgk(f_2)$. Then $\lgk(f_1) \lqs \lgk(f_2)$ by the definition of the minimal $S$-overapproximation of $f_1$.

We next show that $\rgk$ is a functor when it exists. Since $\fbb,\fcb$ are preorders we simply need to show that when $f_1 \lqs f_2$ then $\rgk(f_1) \lqs \rgk(f_2)$. Since $\rgk(f_1) \lqs f_1$ by the definition of the maximal $S$-underapproximation of $f_1$ we have that $\rgk(f_1) \lqs f_2$. Then $\rgk(f_1) \lqs \rgk(f_2)$ by the definition of the maximal $S$-underapproximation of $f_2$.

Next, we will show that $\lgk$ and $\rgk$ are respectively the left and right Kan extensions when they exist. First, by Proposition \ref{proposition:maximal-and-minimal} if $f \in \sofc$ then $f$ must be both the minimal $S$-overapproximation and maximal $S$-underapproximation of $f$. Therefore we have:
\begin{gather*}
    K(f) = \lgk(f) = \rgk(f)
\end{gather*}


Next, consider any functor $M_L: \fbb \rightarrow \fcb$ such that $\forall f \in \sofc, K(f) \lqs M_L(f)$. Since $f =_{S} K(f)$ this implies $f \lqs M_L(f)$ so by the definition of the minimal $S$-overapproximation $\lgk(f) \lqs M_L(f)$.

Next, consider any functor $M_R: \fbb \rightarrow \fcb$ such that $\forall f \in \sofc, M_R(f) \lqs K(f)$. Since $K(f) =_{S} f$ this implies $M_R(f) \lqs f$ so by the definition of the maximal $S$-underapproximation $M_R(f) \leq \rgk(f)$.
\end{proof}

\end{document}